\documentclass{article}
\usepackage{mathrsfs}
\usepackage[utf8]{inputenc}
\usepackage[letterpaper,margin=1.0in]{geometry}

\usepackage{amsmath}\allowdisplaybreaks
\usepackage{amsfonts,bm}
\usepackage{amssymb}
\usepackage{algorithm}
\usepackage{algorithmic}
\usepackage{multirow}
\usepackage{booktabs}
\usepackage{xcolor}

\def\eqref#1{equation~(\ref{#1})}

\def\ceil#1{\left\lceil #1 \right\rceil}
\def\floor#1{\left\lfloor #1 \right\rfloor}
\def\1{\bf{1}}

\newcommand{\Norm}[1]{\left\| #1 \right\|}

\def\fO{{\mathcal{O}}}
\def\fP{{\mathcal{P}}}
\def\fQ{{\mathcal{Q}}}

\def\sB{{\mathbb{B}}}

\def\BE{{\mathbb{E}}}

\def\BR{{\mathbb{R}}}
\def\BS{{\mathbb{S}}}

\usepackage{amsthm}
\theoremstyle{plain}

\makeatletter
\def\Ddots{\mathinner{\mkern1mu\raise\p@
\vbox{\kern7\p@\hbox{.}}\mkern2mu
\raise4\p@\hbox{.}\mkern2mu\raise7\p@\hbox{.}\mkern1mu}}
\makeatother

\makeatletter
\newcommand*{\rom}[1]{\expandafter\@slowromancap\romannumeral #1@}
\makeatother

\usepackage{color}
\usepackage{hyperref}
\usepackage{enumerate}
\usepackage{enumitem}
\usepackage{algorithm}
\usepackage{algorithmic}
\usepackage[numbers,sort,compress]{natbib}
\usepackage{url}
\usepackage{microtype}
\usepackage{hyperref}
\usepackage{url}
\usepackage{microtype}
\usepackage{enumitem}
\usepackage{graphicx}

\usepackage{thmtools}
\usepackage{mathtools}
\usepackage{thm-restate}
\usepackage{nicefrac}
\usepackage{enumitem}
\usepackage{threeparttable}
\usepackage{makecell}

\newtheorem{theorem}{Theorem}
\newtheorem{assumption}{Assumption}
\newtheorem{proposition}{Proposition}
\newtheorem{lemma}{Lemma}
\newtheorem{definition}{Definition}
\newtheorem{remark}{Remark}

\title{Zeroth-Order Nonconvex Nonsmooth Optimization with Heavy-Tailed Noise}

\date{}
\author{Zhuanghua Liu \quad\quad\quad Luo Luo}

\begin{document}

\maketitle

\begin{abstract}
This paper considers the nonconvex nonsmooth problem in which the objective function is Lipschitz continuous.
We focus on the stochastic setting where the algorithm can access stochastic function value evaluations with heavy-tailed noise, which is prevalent in many popular machine learning applications.
We propose a stochastic zeroth-order algorithm that refines the framework of online-to-nonconvex conversion by clipping the two-point gradient estimator.
The theoretical analysis shows that our algorithm can find a $(\delta, \epsilon)$-Goldstein stationary point with zeroth-order oracle complexity of
${\mathcal O}(d^{\frac{p}{2(p-1)}}\delta^{-1}\epsilon^{-\frac{2p-1}{p-1}})$, 
where $d$ is the problem dimension and $p\in(1,2]$ is the order of bounded moments. 
Note that our dependence on dimension $d$ matches the best-known results of stochastic zeroth-order optimization for finding the sub-optimal solution of a stochastic convex nonsmooth problem.
In addition, our dependence on accuracy parameters $\delta$ and $\epsilon$ is consistent with that of the best-known stochastic first-order algorithms for stochastic nonconvex nonsmooth problems.
Finally, we conduct numerical experiments to demonstrate the effectiveness of the proposed method.
\end{abstract}

\section{Introduction}
This paper considers the following unconstrained stochastic optimization problem
\begin{equation}\label{obj}
    \min_{x \in \BR^d} f(x) = \BE_{\xi \sim \fP}[F(x; \xi)],    
\end{equation}
where the stochastic component function $F(x;\xi) \colon \BR^d \to \BR$ is $L(\xi)$-Lipschitz continuous with respect to $x$ but possibly nonconvex and nonsmooth.
Such formulation is very popular in machine learning applications, including statistical learning \citep{fan2001variable,zhang2010analysis,mazumder2011sparsenet}, reinforcement learning \citep{mania2018simple,suh2022differentiable,choromanski2018structured}, deep neural networks \citep{nair2010rectified,glorot2011deep,chen2017zoo}, and simulation optimization \citep{hong2006discrete,nelson2010optimization}.

We focus on stochastic zeroth-order optimization where the algorithm can only access the noisy function evaluations.
The primary motivation for such gradient-free oracles arises in the context of simulation optimization \citep{hong2006discrete,nelson2010optimization}, where gradients (first-order information) are expensive or even impossible to evaluate. 
There has been increasing interest in developing stochastic zeroth-order algorithms for nonconvex nonsmooth optimization under the bounded variance assumption \citep{lin2022gradient,chen2023faster,kornowski2024algorithm,liu2024zeroth}.
In particular, the best known stochastic zeroth-order oracle complexity is achieved by applying the technique of online-to-nonconvex conversion, which attains the optimal dependence on the dimension.

It is worth noting that the popular bounded variance assumption adopted in many existing analyses of stochastic nonconvex nonsmooth optimization is rather restrictive and often fails to hold in practice. 
In contrast, heavy-tailed noise has been widely observed in empirical studies across various real-world applications, including training the machine learning models  \citep{simsekli2019tail,zhang2020adaptive,ahn2023linear,garg2021proximal}.
This phenomenon motivates us to adopt the weaker assumption of $p$-th bounded moment on the noise of function value evaluations. Specifically, we assume formulation (\ref{obj}) holds that
\begin{align*}
    \BE[L(\xi)^p] \leq L^p,
\end{align*}
for some constant $L > 0$, where $p \in (1, 2]$ denotes the order of the moment of the noise.
The special case of $p = 2$ leads to the bounded variance, which has been well studied \citep{ghadimi2013stochastic,balasubramanian2022zeroth}.
For a general $p \in (1,2)$, the variance may be unbounded, giving rise to heavy-tailed noise distributions such as the L\'evy $\alpha$-stable distribution.
In such a case, all existing works on stochastic zeroth-order optimization focus on the convex setting \citep{kornilov2023gradient,kornilov2023accelerated,kornilov2024median},
while there is quite a limited understanding of the nonconvex case, which plays a crucial role in modern machine learning theory.

In this paper, we propose an efficient stochastic zeroth-order stochastic algorithm called \underline{z}eroth-\underline{o}rder \underline{c}lipped \underline{o}nline-t\underline{o}-\underline{n}onconvex (ZOCOON) algorithm for solving stochastic nonconvex nonsmooth optimization under the heavy-tailed noise.
Building upon the framework of online-to-nonconvex conversion \citep{cutkosky2023optimal,kornowski2024algorithm}, our algorithm leverages the gradient clipping on the stochastic zeroth-order gradient estimator to mitigate the heavy-tailed noise.
We establish the theoretical guarantee to show that our ZOCOON attains a $(\delta, \epsilon)$-Goldstein stationary point of the objective within the stochastic zeroth-order oracle complexity of $\fO(d^{\frac{p}{2(p-1)}}\delta^{-1} \epsilon^{-\frac{2p-1}{p-1}})$ both in expectation and with high probability.
Notably, the resulting dependence on the dimension $d$ matches the best-known complexity bounds of stochastic zeroth-order methods for convex nonsmooth optimization.
Moreover, the dependence on the accuracy parameters $\delta$ and $\epsilon$ is consistent with that of the best-known stochastic first-order methods under the heavy-tailed noise~\citep{liu2024high}.
We compare the proposed methods with prior work for stochastic nonconvex nonsmooth optimization in Table~\ref{tbl:zo_res}.
We also show that our results extend naturally to the smooth setting, where the resulting $\epsilon$-dependence matches that of the best-known stochastic first-order methods \citep{hubler2024gradient,liunonconvex2025}.
Finally, we perform numerical experiments to validate the effectiveness of the proposed methods.

\begin{table*}[t]
\caption{
We compare the proposed ZOCOON algorithm with baseline methods by showing their stochastic zeroth-order oracle complexity bounds for achieving a $(\delta, \epsilon)$-Goldstein stationary point.
In the case of $p=2$, the complexity of our ZOCOON matches the best-known results of ZOO2N \citep{kornowski2024algorithm}.
}

\label{tbl:zo_res}\vskip -0.2cm
\begin{center}
\begin{sc}
\begin{tabular}{cccccc}
\toprule
Methods   & Complexity & 
$p$ & References
\\
\midrule
~~~~~~~~GFM~~~~~~~~  & $\fO\left(d^{\frac{3}{2}} \delta^{-1} \epsilon^{-4}\right)$ & 2 & \cite{lin2022gradient}
  \\\addlinespace
GFM\textsuperscript{+}   & $\fO\left(d^{\frac{3}{2}} \delta^{-1} \epsilon^{-3}\right)$ & 2 & \cite{chen2023faster} 
  \\\addlinespace
  ZOO2N   & $\fO\left(d \delta^{-1} \epsilon^{-3}\right)$ & 2 & \cite{kornowski2024algorithm}
   \\\addlinespace
 ZOCOON    &  ~~~~$\fO\left(d^{\frac{p}{2(p-1)}}\delta^{-1}\epsilon^{-\frac{2p-1}{p-1}}\right)$~~~~ & ~~(1,2]~~ & Theorem~\ref{thm:high_prob_res} \\\addlinespace
\bottomrule    
\end{tabular}
\end{sc}
\end{center}
\end{table*}

\paragraph{Paper Organization}
In Section~\ref{sec:related_work}, we present a literature review on the convergence analysis of the nonconvex nonsmooth optimization problem and stochastic optimization with heavy-tailed noise.
In Section~\ref{sec:preliminaries}, we formalize the notations and assumptions of our problem and introduce the stationarity of the nonconvex nonsmooth optimization problem.
In Section~\ref{sec:methodology}, we propose our ZOCOON algorithm and provide its convergence analysis.
In Section~\ref{sec:experiments}, we conduct numerical experiments to demonstrate the improved effectiveness of the proposed methods.
Finally, we conclude this work in Section~\ref{sec:conclusion}.

\section{Related Work}\label{sec:related_work}


For general nonconvex nonsmooth optimization,
\citet{zhang2020complexity} introduce the notion of $(\delta, \epsilon)$-Goldstein stationary points to describe the approximate stationarity of the Lipschitz continuous objective.
They also provide an algorithm based on the non-standard oracle of Hadamard directional derivative to find $(\delta, \epsilon)$-Goldstein stationary points with a dimension-independent convergence rate.
The subsequent works apply the perturbation techniques to design non-asymptotic convergent algorithms with the standard first-order oracle~\citep{davis2021gradient, tian2022finite}.
Moreover, \citet{tian2022no, jordan2023deterministic} proved that randomization is necessary to achieve dimension-free complexity bounds for the general nonconvex nonsmooth optimization.
Building on these insights, \citet{cutkosky2023optimal} proposed an online-to-nonconvex framework that achieves the optimal first-order oracle complexity bound for finding $(\delta, \epsilon)$-Goldstein stationary points.
Besides, another line of work consider finding the near approximate stationary of the weakly convex objectives \citep{davis2019proximally, davis2019stochastic}.
However, a broad class of nonconvex nonsmooth functions, including the objective function in training deep neural networks with ReLU activations, does not satisfy the weak convexity.

In the scenario that the first-order information is unavailable,
\citet{lin2022gradient} designed the gradient-free method by considering the smoothed surrogate of the objective, achieving the $(\delta, \epsilon)$-Goldstein stationary points with the stochastic zeroth-order oracle complexity of~$\fO(d^{\frac{3}{2}}\delta^{-1}\epsilon^{-4})$.
Later, \citet{chen2023faster} introduced the variance reduction to attain an improved upper bound of~$\fO(d^{\frac{3}{2}}\delta^{-1}\epsilon^{-3})$.
Furthermore, \citet{kornowski2024algorithm} extended the framework of online-to-nonconvex conversion  \citep{cutkosky2023optimal} to achieve the complexity bound of $\fO(d\delta^{-1}\epsilon^{-3})$, which is optimal with respect to the dimension.
It is worth noting that all existing zeroth-order optimization methods for finding $(\delta, \epsilon)$-Goldstein stationary points require the boundedness of the second-order moment, while the more general setting of the heavy-tailed noise has not been studied.


For stochastic nonconvex optimization with heavy-tailed noise, most of the existing work focus on the smooth case.
For example, \citet{zhang2020adaptive, sadiev2023high, nguyen2023high} provided the in-expectation and the high-probability convergence guarantees for clipped stochastic gradient descent.
Recently, \citet{hubler2024gradient, liunonconvex2025,sun2024gradient} investigated the convergence of normalized stochastic gradient descent methods and showed this class of methods attain the optimal stochastic first-order oracle complexity for finding $(\delta, \epsilon)$-Goldstein stationary points.
For the nonsmooth setting, \citet{liu2024high} proposed an efficient stochastic first-order method by combining the gradient clipping with the online-to-nonconvex framework \citep{cutkosky2023optimal}, while the optimality for the nonconvex nonsmooth optimization with heavy-tailed noise is still an open problem even for the first-order oracle.
Additionally, existing zeroth-order nonsmooth optimization methods for heavy-tailed noise are limited to convex problems \citep{kornilov2023accelerated}.

\section{Preliminaries}\label{sec:preliminaries}

We impose the following assumptions for our stochastic nonconvex nonsmooth optimization problem (\ref{obj}).

\begin{assumption}\label{asm:Lip}
We suppose the stochastic component $F(\cdot, \xi)$ in problem (\ref{obj}) is $L(\xi)$-Lipschitz for a given $\xi$, i.e., it holds that
\begin{align*}
|F(x; \xi) - F(y; \xi)| \leq L(\xi) \Norm{x - y}
\end{align*}
for all $x,y\in\BR^d$, where $L(\xi)$ has the bounded $p$th moment such that 
$\BE_{\xi} [L(\xi)^p] \leq L^p$ for some $L>0$ and $p \in (1,2]$.
\end{assumption}

In the case of $p<2$, the noise characterized by Assumption~\ref{asm:Lip} is said to be heavy-tailed, and its variance may be infinite. 
In the case of $p=2$, the assumption reduces to the standard assumption of bounded second-order moment.
Based on Assumption \ref{asm:Lip}, we can show that the objective  $f:\BR^d\to\BR$ is $L$-Lipschitz on $\BR^d$ by using Jensen's inequality.

In addition, we assume that the objective function satisfies the following mild regularity condition and admits a proper lower bound, which are standard in the study of nonconvex nonsmooth optimization \citep{cutkosky2023optimal,kornowski2024algorithm,liu2024high}.

\begin{assumption}\label{asm:well}
    We suppose the objective function $f(\cdot)$ in problem (\ref{obj}) is well-behaved, i.e., it holds 
    \begin{align*}
        f(y) - f(x) = \int_{0}^{1} \langle \nabla f(x + t(y-x)), y-x \rangle\,{\rm d}t.
    \end{align*}
    for all $x, y \in \BR^d$.
\end{assumption}

\begin{assumption}\label{asm:lower}
    The objective $f(\cdot)$ in problem (\ref{obj}) is lower-bounded, i.e., it holds $f^* \coloneqq \inf_{x \in \BR^d} f(x) > - \infty$.
\end{assumption}

We define the generalized directional derivative and the generalized gradient for the Lipschitz function as follows. 

\begin{definition}[\citet{clarke1990optimization}]
    Given a point $x \in \BR^d$ and a direction $v \in \BR^d$, the generalized directional derivative of a Lipschitz function $f:\BR^d\to\BR$ is defined as 
    \begin{align*}
        D f(x; v) \coloneqq \lim \sup_{y\to x, t \downarrow 0} \frac{f(y + tv) - f(y)}{t}.
    \end{align*}
    In addition, the Clarke subdifferential of $f:\BR^d\to\BR$ is defined as $\partial f(x) \coloneqq \{g \in \BR^d \colon \langle g, v \rangle \leq D f(x; v), \forall v \in \BR^d \}$ 
    and each $g \in \partial f(x)$ is called a generalized gradient of $f$.
\end{definition}

Given the definitions of the Clarke subdifferential and generalized gradients, the convergence criterion for solving the nonconvex nonsmooth optimization problem $\min_{x\in \BR^d} f(x)$ can be characterized by finding an $\epsilon$-Clarke stationary point~$x$ of the function $f$ such that
\begin{equation*}
    \min\{ \Norm{g}\colon g\in \partial f(x)\} \leq \epsilon.
\end{equation*}
However, \citet{zhang2020complexity} demonstrated that this task is intractable.
To address this issue, they introduced a refined notion of approximate stationarity based on the Goldstein $\delta$-subdifferential, 
which is a convex hull of the generalized gradients at the points in the $\delta$-neighbourhood.

\begin{definition}[\citet{goldstein1977optimization}]
For a given Lipschitz continuous function $f \colon \BR^d \to \BR$ and some $\delta > 0$, 
the Goldstein $\delta$-subdifferential of $f$ at the point $x\in\BR^d$ is defined as
    \begin{equation*}
        \partial_{\delta} f(x) \coloneqq {\rm conv}\left(\cup_{y \in \sB(x, \delta)} \partial f(y) \right),
    \end{equation*}
    where
    $\sB(x, \delta) = \{y : \Norm{y - x} \leq \delta \}$.
\end{definition}

Accordingly, a refined notion for approximate stationarity in nonconvex nonsmooth problem is defined as follows.

\begin{definition}[\citet{zhang2020complexity}]
   Given a Lipschitz continuous function $f \colon \BR^d \to \BR$, a point $x \in \BR^d$, and $\delta > 0$, we denote 
   \begin{align*}
       \Norm{\nabla f(x)}_{\delta} \coloneqq \min\{\Norm{g}: g \in \partial_\delta f(x)\}.
   \end{align*}
   We say the point $x$ is a $(\delta, \epsilon)$-Goldstein stationary point of $f\colon \BR^d \to \BR$ if it holds 
   \begin{align*}
       \Norm{\nabla f(x)}_{\delta} \leq \epsilon.
   \end{align*}
\end{definition}
The task of finding a $(\delta, \epsilon)$-Goldstein stationary point is a standard criterion in the non-asymptotic convergence analysis for nonconvex nonsmooth optimization.

The randomized smoothing is a widely used technique in nonsmooth analysis \citep{duchi2012randomized} and zeroth-order optimization \citep{nesterov2017random}.
Formally speaking, given a $L$-Lipschitz function $f$, we define its smoothed surrogate function as 
\begin{align}\label{eq:f_delta}
    f_{\delta}(x) = \BE_{u \sim \fQ}[f(x + \delta u)],
\end{align}
where $\fQ$ is a uniform distribution on $\sB(0, \delta) = \{y:\Norm{y} \leq \delta \}$.
Additionally, the function $f_\delta$ has the following properties.

\begin{lemma}[\citet{yousefian2012stochastic}]
    Suppose the function $f:\BR^d\to\BR$ is $L$-Lipschitz continuous, then its smoothed surrogate $f_{\delta}:\BR^d\to\BR$ holds that 
    \begin{enumerate}
        \item $|f_{\delta}(x) - f(x)| \leq \delta L$ for all $x\in\BR^d$; 
        \item  $f_{\delta}$ is differentiable everywhere and $L$-Lipschitz with $({c L \sqrt{d}}/{\delta})$-Lipschitz gradient for some constant $c > 0$;
        \item $\nabla f_{\delta}(x) \in \partial_{\delta} f(x)$ for all $x \in \BR^d$.
    \end{enumerate}
    \label{smooth_lemma}
\end{lemma}

It is worth noting that Lemma~\ref{smooth_lemma} implies the problem of finding a $(\delta,\epsilon)$-Goldstein stationary point of the Lipschitz continuous function $f$ can be reduced to 
finding an $\epsilon$-stationary point of the smooth function $f_\delta$.

\section{Methodology}\label{sec:methodology}
In this section, we introduce a stochastic zeroth-order method for solving the nonconvex nonsmooth optimization problem under heavy-tailed noise. We also provide convergence guarantees for our method both in expectation and with high probability.

\subsection{The Algorithm and Main Results}\label{sec:algorithm}













\begin{algorithm}[tb]
    \caption{Zeroth-Order Clipped Online-to-Nonconvex}
    \label{alg:o2n}
    \textbf{Input}: Initial point $x_0$, number of restarts $K$, round length~$T$, clipping parameter $\tau$, smoothing parameter $\delta' > 0$, and domain radius $D$
   \begin{algorithmic}[1] 
        \STATE Set $M = KT$, $\eta = D / \tau$, $\Delta_1 = 0$.\\[0.12cm]
        \FOR{$n = 1, \dots, M$}
        \STATE Update $x_n = x_{n-1} + \Delta_n$.\\[0.12cm]
        \STATE Sample $s_n \in [0, 1]$.\\[0.12cm]
        \STATE Update $w_n = x_{n-1} + s_n \Delta_n$.\\[0.12cm]
        \STATE Sample $\xi_n\sim \fP$.\\[0.12cm]
        \STATE Set $g_n = {\rm GradientEstimator}(w_n, \delta', \xi_n)$.\\[0.12cm]
        \STATE Update $\hat{g}_n = {\rm clip}(g_n, \tau)$.\\[0.12cm]
        \STATE Update $\Delta_{n+1} = \Pi_{\sB(0, D)}[\Delta_n - \eta \hat{g}_n]$.\\[0.12cm]
        \ENDFOR\\[0.12cm]
        \STATE Update $w_t^k = w_{(k-1)T + t}$ for $k \in [K]$ and $t \in [T]$.\\[0.12cm]
        \STATE Update ${\Bar{w}^k = \frac{1}{T} \sum_{t=1}^T w_t^k}$.\\[0.12cm]
        \STATE \textbf{Return:} $w_{\rm out} \sim {\rm Unif}(\Bar{w}^1, \dots, \Bar{w}^K)$.    \end{algorithmic}
\end{algorithm}

In this section, we present a stochastic zeroth-order algorithm for nonconvex nonsmooth optimization under heavy-tailed noise, and establish its convergence guarantees.
We first present the details of our method, termed the \underline{Z}eroth-\underline{O}rder \underline{C}lipped \underline{O}nline-t\underline{o}-\underline{N}onconvex (ZOCOON) algorithm, for solving the objective function~(\ref{obj}) in Algorithm~\ref{alg:o2n}. The design of our method follows the online-to-nonconvex framework proposed by~\citep{cutkosky2023optimal}, with the key difference lying in the gradient computation step.
Since the exact gradient oracle is unavailable in our setting, our ZOCOON employs a two-point stochastic gradient estimator (Algorithm \ref{alg:grad}) to establish the gradient estimator, which is defined as follows.

\begin{definition}
    Given a stochastic function  $F(\cdot; \xi)\colon \BR^d \to \BR$, we denote its zeroth-order stochastic gradient estimator at the point $x \in \BR^d$ by
    \begin{equation}
        \hat{g}(x; w, \xi) = \frac{d}{2 \delta'} (F(x + \delta' w; \xi) - F(x - \delta' w; \xi))w,
        \label{eq:two_point}
    \end{equation}
    where $w$ is sampled from uniform distribution on the unit sphere $\BS^{d-1}=\{y\in\BR^{d}:||y||=1\}$.
\end{definition}
\begin{remark}
The two-point estimator~(\ref{eq:two_point}) provides an unbiased stochastic estimate for the gradient of the smoothed surrogate function~$f_{\delta'}(x)$ \citep{agarwal2010optimal}.    
\end{remark}

To mitigate the effect of heavy-tailed noise, we additionally incorporate a clipping step into the framework of online-to-nonconvex conversion (Line 8 of Algorithm \ref{alg:o2n}).
Specifically, the clipping operator with the parameter $\tau$ is defined as
\begin{align*}
    {\rm clip}(g, \tau) =
    \begin{cases}
        \min\left\{1, \dfrac{\tau}{\Norm{g}}\right\}g, & g\neq0, \\
        0, & g=0,
    \end{cases}
\end{align*}
which prevents excessively effect of large stochastic gradients arising from heavy-tailed noise.


\begin{algorithm}[tb]
    \caption{ZO-GradientEstimator($x$, $\delta'$, $\xi$)}
    \label{alg:grad}
    \textbf{Input}: Point $x \in \BR^d$, smoothing parameter $\delta' > 0$, and random variable $\xi$
   \begin{algorithmic}[1] 
        \STATE Sample $w \sim {\rm Unif}(\BS^{d-1})$\\[0.15cm]
        \STATE Evaluate $F(x + \delta' w; \xi)$ and $F(x - \delta' w ; \xi)$\\[0.15cm]
        \STATE Update $g = \frac{d}{2 \delta'}(F(x + \delta' w; \xi) - F(x - \delta' w; \xi))w$\\[0.15cm]
        \STATE \textbf{Return:} $g$.
\end{algorithmic}
\end{algorithm}

We now present our main theoretical result, which characterizes the convergence guarantee in-expectation for the proposed ZOCOON.

\begin{theorem}
Under Assumptions \ref{asm:Lip}, \ref{asm:well}, and \ref{asm:lower},
we run our ZOCOON (Algorithm \ref{alg:o2n}) by implementing the subroutine of ${\rm GradientEstimator}(w_n, \delta', \xi_n)$ by Algorithm \ref{alg:grad}
and taking 
\begin{align*}
    &   T = \min\left(\ceil{\left(\frac{\delta M L (d^{\frac{p}{2}} + 1)^{\frac{1}{p}}}{2\delta L + 2 \Delta}\right)^{\frac{p}{2p-1}}}, \frac{M}{2}\right),\\ 
        & K = \floor{\frac{M}{T}}, ~ D=\frac{\delta}{2T}, ~ \tau = T^{\frac{1}{p}} L (d^{\frac{p}{2}} + 1)^{\frac{1}{p}}, ~\delta'=\frac{\delta}{2}
\end{align*}
for some $\delta,\epsilon\in(0,1)$ and $\Delta\geq f(x_0)-f^*$,
then the output satisfies  that
\begin{align*}
    & \BE [\Norm{\nabla f(w_{\rm out})}_{\delta}] \\ 
    \leq &\max\!\left(\frac{18 (\Delta\!+\!\delta L )^{\frac{p-1}{2p-1}} (d^{\frac{p}{2}} L^p\!+\!L^p)^{\frac{1}{2p - 1}}}{(\delta M)^{\frac{p-1}{2p-1}}}, \frac{16L (d^{\frac{p}{2}}\!+\!1)^{\frac{1}{p}}}{M^{\frac{p-1}{p}}}\right) \\
    & + \frac{2\Delta + 2 \delta L}{\delta M}
\end{align*}
\label{thm:in_expect}
\end{theorem}

Theorem~\ref{thm:in_expect} implies running ZOCOON by taking 
\begin{align}
     M = \fO\left(\frac{ (\Delta + \delta L ) d^{\frac{p}{2(p-1)}} L^{\frac{p}{p-1}} }{\delta \epsilon^{\frac{2p-1}{p-1}}}\right),
\label{number_of_sample_expect}
\end{align}
can guarantee the output satisfies $\BE[\Norm{\nabla f(w_{\rm out})}_{\delta}] \leq \epsilon$,
i.e., the point $w_{\mathrm{out}}$ is a $(\delta, \epsilon)$-Goldstein stationary point of the objective function $f$ in expectation.
In the special case of bounded second-order moment (i.e., $p=2$), the resulting complexity of ZOCOON matches that of the best-known stochastic zeroth-order methods provided by~\citet{kornowski2024algorithm}.

Besides the above in-expectation results for achieving the small $\Norm{\nabla f(w_{\rm out})}_{\delta}$, 
we also establishes a high-probability theoretical guarantee  as follows.

\begin{theorem}
    Under Assumptions \ref{asm:Lip}, \ref{asm:well}, and \ref{asm:lower},
    we run our ZOCOON (Algorithm \ref{alg:o2n}) by implementing the subroutine of ${\rm GradientEstimator}(w_n, \delta', \xi_n)$ by Algorithm \ref{alg:grad}
    and taking 
    \begin{align*}
        & T = \min\left(\ceil{\left(\frac{A M \delta}{2\delta L + 2 \Delta}\right)^{\frac{p}{2p-1}}}, \frac{M}{2}\right), ~ K = \floor{\frac{M}{T}},\\
        & D = \frac{\delta}{2 T}, ~~ \tau = T^{\frac{1}{p}} L (d^{\frac{p}{2}} + 1)^{\frac{1}{p}}  \left(\log \left(\frac{2K}{q}\right)\right)^{-\frac{1}{p}},
    \end{align*}
    for some $\delta,\epsilon\in(0,1)$ and $\Delta\geq f(x_0)-f^*$,
    then with probability at least $1-q$, the output satisfies that
    \begin{align*}
        &\Norm{\nabla f(w_{\rm out})}_\delta \\
        \leq & \frac{4 \delta L \! + \! 4 \Delta}{\delta M} \! + \!\frac{2 B}{\sqrt{M}} \! + \! \max\left(\frac{8 A^{\frac{p}{2p-1}} (\Delta \! + \! \delta L)^{\frac{p-1}{2p-1}}}{(\delta M)^{\frac{p-1}{2p-1}}}, \frac{4A}{M^{\frac{p-1}{p}}}\right),
    \end{align*}
    where
    \begin{align*}
        & A = L (d^{\frac{p}{2}} + 1)^{\frac{1}{p}} \left(\log\left(\frac{2K}{q}\right)\right)^{-\frac{1}{p}}\left(8 + \frac{29}{2} \log\left(\frac{2K}{q}\right) \right), \\ 
        & B = \sqrt{8 \log\left(\frac{6}{q}\right)}.
    \end{align*}
    \label{thm:high_prob_res}
\end{theorem}

Theorem \ref{thm:high_prob_res} implies that we can achieve a $(\delta, \epsilon)$-Goldstein stationary point with high probability by taking
\begin{align*}
    M = \fO\left(\frac{d^{\frac{p}{2(p-1)}}L^{\frac{p}{p-1}}(\Delta + \delta L)}{\delta \epsilon^{\frac{2p-1}{p-1}}}\log\left(\frac{2K}{q}\right)\right),
\end{align*}
which only include an additional logarithmic factor $\log(2K/q)$ in the setting of $M$ compared with that of the in-expectation result shown in \eqref{number_of_sample_expect}.
Recently, \citet{liu2025online} established a lower bound of $\Omega(\gamma L^{\frac{p}{p-1}} \delta^{-1} \epsilon^{-\frac{2p-1}{p-1}})$ for finding $(\delta, \epsilon)$-Goldstein stationary point for nonconvex nonsmooth problem with heavy-tailed noise.
By comparing the oracle complexity of ZOCOON in~(\ref{number_of_sample_expect}) with the lower bound, we observe that the dependencies on $L$, $\delta$, and $\epsilon$ of the ZOCOON method match the lower bound.

\subsection{Convergence Analysis}

In this subsection, we present the in-expectation convergence analysis of the ZOCOON algorithm in Theorem~\ref{thm:in_expect}.
We perform the analysis by starting with the following lemma that establishes the relationship between the $(\delta, \epsilon)$-Goldstein stationary point of the nonsmooth function $f$ and that of its smoothed surrogate function as follows.  
\begin{lemma}[\citet{kornowski2024algorithm}]
    For any $\delta, \gamma \geq 0$, it holds $\partial_\gamma f_{\delta}(x) \subseteq \partial_{\delta + \gamma} f(x)$, 
    Hence, if $x$ is an $(\gamma, \epsilon)$-Goldstein stationary point of $f_{\delta}$, then it is a $(\delta + \gamma, \epsilon)$-Goldstein stationary point of $f$.
    \label{lemma:nonsmooth_smooth_conversion}
\end{lemma}

Furthermore, the two-point zeroth-order gradient estimator satisfies the following properties.
\begin{lemma}[\citet{kornilov2023accelerated}]
    Under Assumption \ref{asm:Lip}, let 
    \begin{align*}
        g = \frac{d}{2 \delta} (F(x + \delta w; \xi) - F(x - \delta w;\xi))w,
    \end{align*}
    where $w$ is uniformly sampled from the unit sphere and $\delta>0$. Then it holds that
    \begin{align*}
        \BE[g] = \nabla f_{\delta}(x),
    \end{align*}
    and 
    \begin{align*}
        \BE[\Norm{g - \nabla f_{\delta}(x)}^p] \leq \left(\frac{\sqrt{d} L}{2^{\frac{1}{4}}}\right)^{p}.
    \end{align*}
    \label{lemma:zero_order_prop}
\end{lemma}

We now introduce the following complexity lemma for the clipped online-to-nonconvex algorithm with a general gradient estimator (Algorithm~\ref{alg:o2n}).

\begin{lemma}
    Let $\delta', \epsilon \in (0,1)$ and $h: \BR^d \to \BR$ be~an $L$-Lipschitz function such that $h(x_0)-\inf_{x \in \BR^d} h(x)\leq \Delta_h$ for a given $x_0\in\BR^d$.
    We run Algorithm \ref{alg:o2n} to solve the problem $\min_{x\in\BR^d}h(x)$ with the initial point $x_0$
    and the unbiased stochastic gradient estimator  
    ${\rm GradientEstimator}(x, \delta', \xi)$ such that    
    $\BE[{\rm GradientEstimator}(x, \delta', \xi) ] = \nabla h(x)$ and $\BE[\Norm{{\rm GradientEstimator}(x, \delta', \xi) - \nabla h(x)}^p] \leq \sigma^p$,
then we take the following parameters as
\begin{align*}
    &  T = \min\left(\ceil{\left(\frac{\delta M (\sigma^p + L^p)^\frac{1}{p}}{\Delta_h}\right)^{\frac{p}{2p-1}}}, \frac{M}{2}\right),\\
    &\tau = T^{\frac{1}{p}}(\sigma^p + L^p)^{\frac{1}{p}},~K = \floor{\frac{M}{T}},~D=\frac{\delta'}{T},
\end{align*}
ensures that
\begin{align}\label{inexpect_converg_main}
\begin{split}
   & \BE[\Norm{\nabla h(w_{\rm out})}_{\delta'}]\\
   \leq &\max\left(\frac{9 \Delta_h^{\frac{p-1}{2p-1}} (\sigma^p + L^p)^{\frac{1}{2p - 1}}}{(\delta' M)^{\frac{p-1}{2p-1}}}, \frac{16(\sigma^p + L^p)^\frac{1}{p}}{M^{\frac{p-1}{p}}}\right) + \frac{\Delta_h}{\delta' M}.
\end{split}
\end{align}
\label{thm:o2n_expect}
\end{lemma}

We then present the proof of Theorem \ref{thm:in_expect} as follows.
\begin{proof}
According to Lemma~\ref{lemma:zero_order_prop}, Algorithm \ref{alg:grad} with $\xi\sim\fP$ will output the gradient estimate
\begin{align*}
g = \frac{d}{2 \delta'}(F(x + \delta' w; \xi) - F(x - \delta' w; \xi))w    
\end{align*}
that satisfies $\BE[g] = \nabla f_{\delta'}(x)$ and $\BE[\Norm{g-\nabla f_{\delta'}(x)}^p]= \fO(d^{\frac{p}{2}}L^p)$.

Then, running Algorithm~\ref{alg:o2n} by taking  Algorithm~\ref{alg:grad} as the subroutine  
${\rm GradientEstimator}(x, \delta', \xi)$, 
which satisfies the conditions
$\BE[{\rm GradientEstimator}(x, \delta', \xi) ] = \nabla h(x)$ and 
\begin{align*}
     \BE[\Norm{{\rm GradientEstimator}(x, \delta', \xi) - \nabla h(x)}^p] \leq \sigma^p,
\end{align*}
where $h = f_{\delta'}$ and $\sigma = \fO(\sqrt{d} L)$.
In the view of Lemma~\ref{thm:o2n_expect} by taking $\delta'=\delta/2$, we have
\begin{align}
\begin{split}
   & \BE[\Norm{\nabla f_{\delta/2}(w_{\rm out})}_{\delta/2}]\\
   \leq &\max\left(\frac{9 (\Delta_{f_{\delta/2}})^{\frac{p-1}{2p-1}} (d^{\frac{p}{2}} L^p \!+\! L^p)^{\frac{1}{2p - 1}}}{(\delta M/2)^{\frac{p-1}{2p-1}}}, \frac{16(d^{\frac{p}{2}} L^p \!+\! L^p)^\frac{1}{p}}{M^{\frac{p-1}{p}}}\right)\! \\ 
   & + \frac{2\Delta_{f_{\delta/2}}}{\delta M} 
   \coloneqq  \epsilon',
\end{split}
\end{align}
where $\Delta_{f_{\delta/2}}=f_{\delta/2}(x_0) - \inf_{x \in \BR^d} f_{\delta/ 2}(x)$.
Hence, the output $w_{\rm out}$ generated from Algorithm \ref{alg:o2n} is a $(\delta / 2, \epsilon')$ Goldstein stationary point of $f_{\delta/ 2}$.
Consequently, Lemma \ref{lemma:nonsmooth_smooth_conversion} implies the output $w_{\rm out}$ is a $(\delta, \epsilon')$-Goldstein stationary point of $f$, i.e.,
\begin{align}\label{inexpect_converg_main2}
   & \BE[\Norm{\nabla f(w_{\rm out})}_{\delta}]
   \leq  \epsilon'.
\end{align}
Comparing equations (\ref{inexpect_converg_main2}) with the desired result of Theorem~\ref{thm:o2n_expect}, we only need to show that $\Delta_{f_{\delta/2}}\leq\delta L + \Delta$.
This can be verified as follows
\begin{align*}
\Delta_{f_{\delta/2}} 
= &f_{\delta/2}(x_0) -  f_{\delta/2}^*\\
= & (f_{\delta/2}(x_0) - f(x_0)) + (f(x_0)- f^*) + (f^* -  f_{\delta/2}^*) \\
\leq &  \frac{\delta L}{2} + \Delta + \frac{\delta L}{2} = \delta L + \Delta,
\end{align*}
where the inequality is based on result $\Norm{f - f_{\delta/2}}_\infty \leq {\delta L}/{2}$ from Lemma~\ref{smooth_lemma} and the setting of $\Delta$.
\end{proof}

Following the above proof for the in-expectation convergence guarantee, we can also establish the high-probability convergence guarantee show in Theorem \ref{thm:high_prob_res}.
Specifically, we first derive a high-probability result for the clipped online-to-nonconvex framework with a general gradient estimator, then apply Lemmas~\ref{lemma:nonsmooth_smooth_conversion} and~\ref{lemma:zero_order_prop} to show our Algorithm \ref{alg:o2n} by taking Algorithm \ref{alg:grad} as the subroutine can achieve an desired approximate Goldstein stationary point.
We defer the detailed proof to Appendix \ref{appendix:high_prob_res}.

\begin{figure*}[!ht]
\centering
\begin{tabular}{ccc}
\includegraphics[scale=0.24]{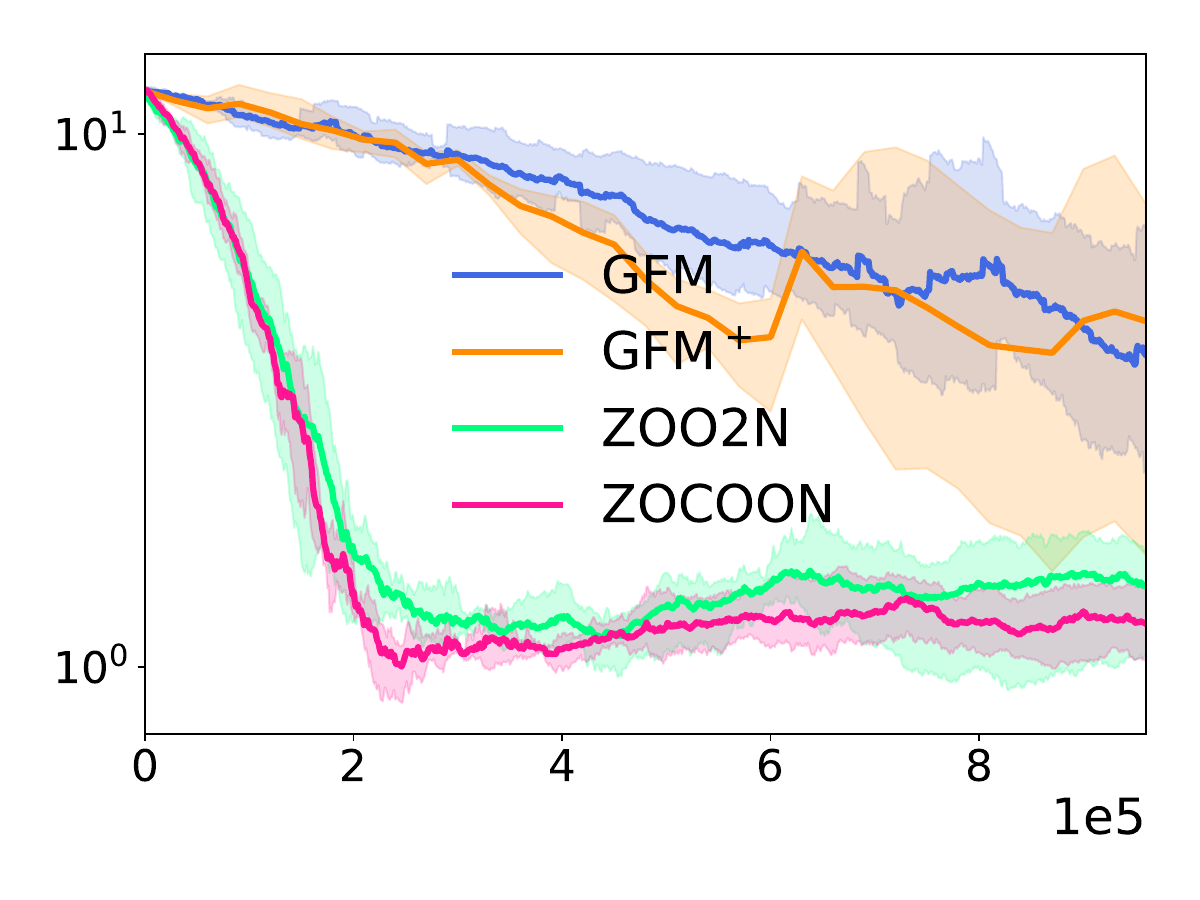} & \includegraphics[scale=0.24]{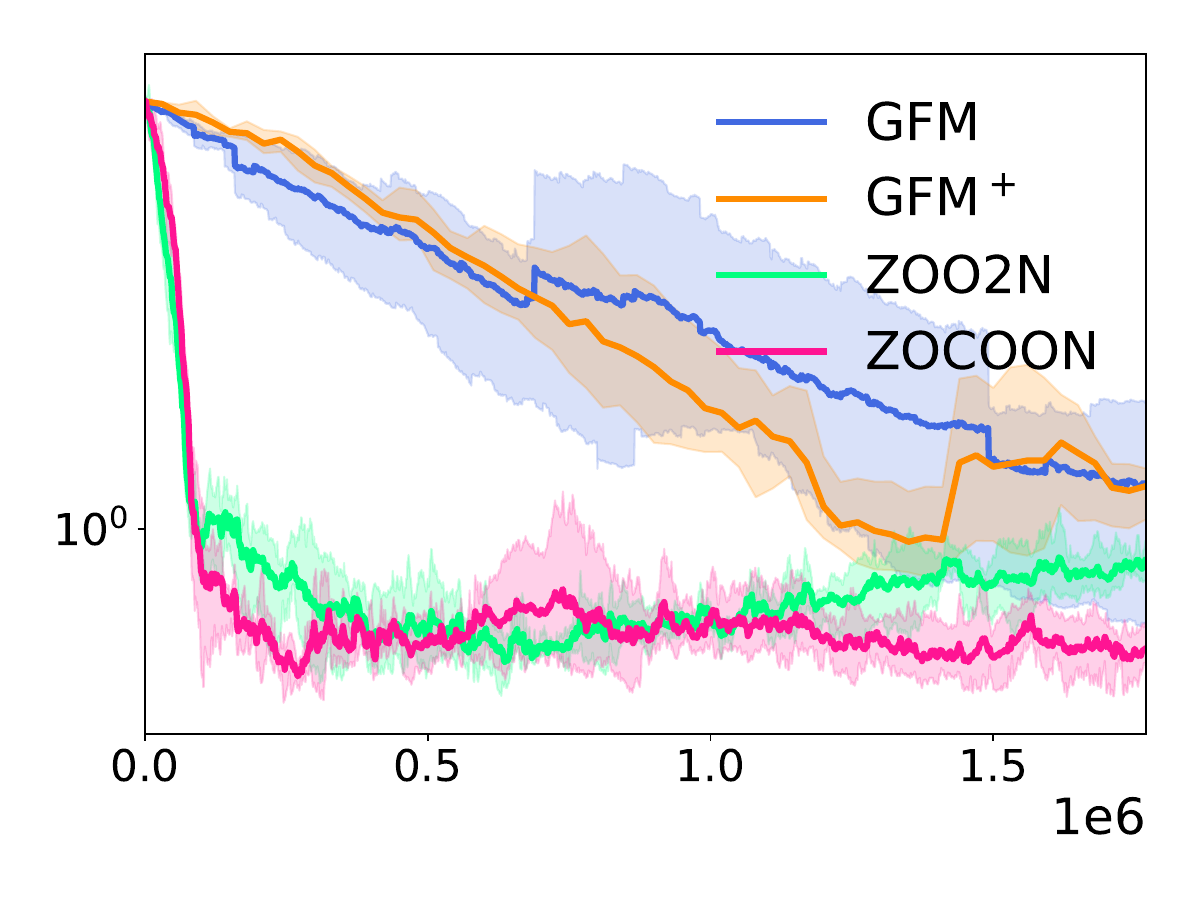} & \includegraphics[scale=0.24]{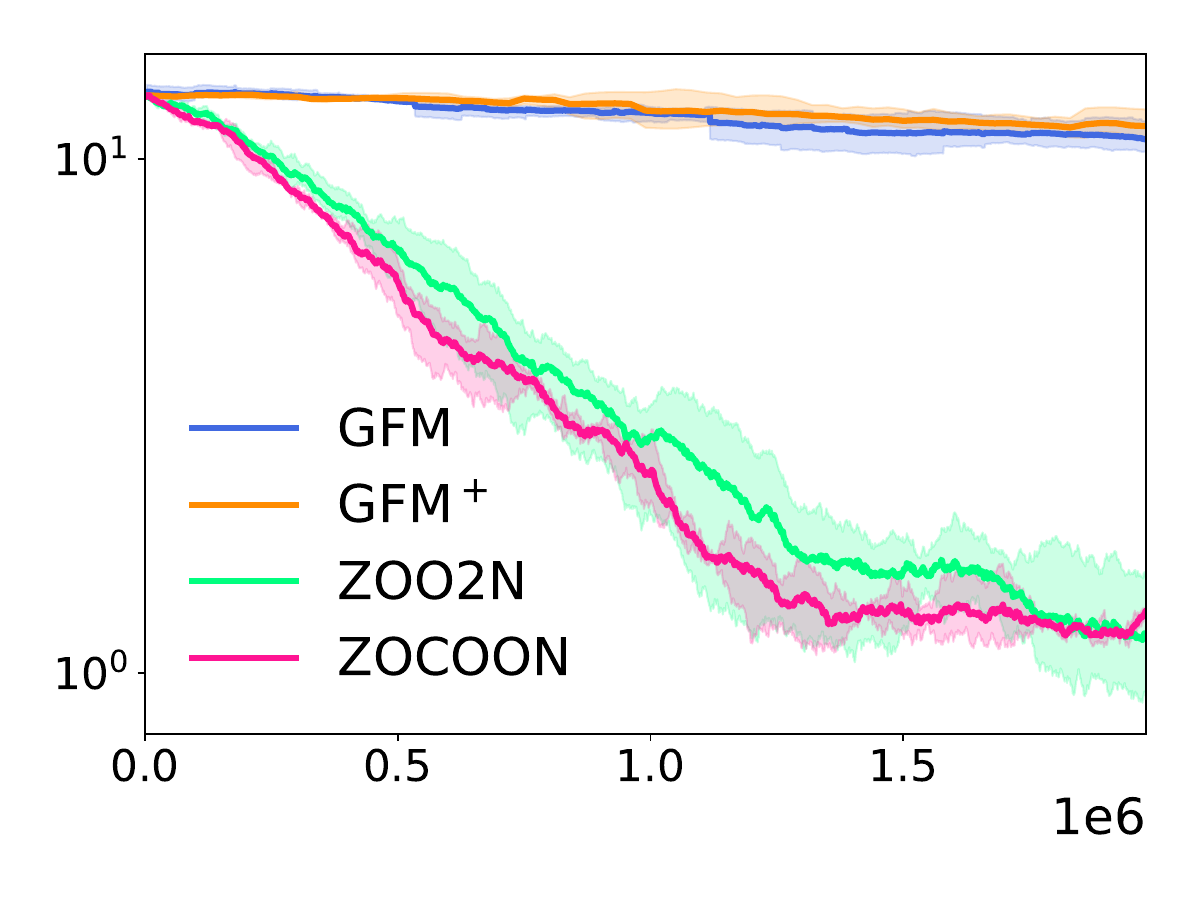} \\
(a)  a9a     &  (b) covtype & (c) w8a  
\end{tabular}
\caption{For nonconvex penalized SVM, we present the results for zeroth-order function query complexity vs. loss on several real-world datasets. The result for each method is averaged over 10 independent runs.}
\label{fig:reg_svm}
\end{figure*}
\vspace{0.3cm}

\subsection{Nonconvex Smooth Guarantee}
In this subsection, we consider extending our result for nonsmooth optimization to smooth optimization.
The following proposition shows that if $x$ is a $(\delta, \epsilon)$-Goldstein stationary point of the smooth function $f$, then the norm of $\nabla f(x)$ can be bounded by 
$\delta$ and $\epsilon$.

\begin{proposition}[\citet{cutkosky2023optimal}]
    Suppose that $f$ is $H$-smooth and $x$ is a $(\delta, \epsilon)$-Goldstein stationary point of the function $f$.
    Then $\Norm{\nabla f(x)} \leq \epsilon + H \delta$.
    \label{prop:station_relation}
\end{proposition}

Recall that Theorem~\ref{thm:in_expect} implies that we can  obtain a $(\delta, \epsilon)$-Goldstein stationary point with $\fO(d^{\frac{p}{2(p-1)}} \delta^{-1} \epsilon^{-\frac{2p-1}{p-1}})$ stochastic zeroth-order oracle calls. 
Therefore, applying Proposition~\ref{prop:station_relation}, with $\delta = \epsilon / H$ yields an $\epsilon$-stationary point of the smooth function within $\fO(d^{\frac{p}{2(p-1)}}  \epsilon^{-\frac{3p-2}{p-1}})$ oracle calls.
The above dependence on $\epsilon$ matches that in the oracle complexity of optimal stochastic first-order methods for nonconvex smooth optimization \citep{hubler2024gradient, liunonconvex2025,zhang2020adaptive}.

\section{Numerical Experiments}\label{sec:experiments}

In this section, we conduct numerical experiments to validate the effectiveness of our approach.
We consider a nonconvex penalized SVM with a capped-$\ell_1$ regularizer \citep{zhang2010nearly}.
The objective is to train a binary classifier $x \in \BR^d$ on a dataset ${(a_i, b_i)}_{i=1}^n$, where $a_i \in \BR^d$ represents the feature vector of the $i$-th sample and $b_i \in \{1, -1\}$ denotes its corresponding label.
To model heavy-tailed noise, we add an additional term $\langle \xi, x \rangle$, where $\xi$ is drawn from a Pareto distribution with the shape parameter chosen as $1.5$.
Formally, the problem can be expressed as the following nonconvex nonsmooth optimization formulation:
\begin{align*}
    \min_{x \in \BR^d} f(x) \coloneqq \frac{1}{n} \sum_{i=1}^n l(b_i a_i^{\top} x) + r(x) + \langle \xi, x \rangle,
\end{align*}
where $l(x) = \max\{1-x, 0\}$, $r(x) = \lambda \sum_{j=1}^d \min\{|x_j|, \alpha\}$ and $\lambda, \alpha > 0$ are hyperparameters.
We take $\lambda = 10^{-5}/n$ and $\alpha = 2$ in our experiments.

We compare the proposed ZOCOON method with baselines including GFM~\citep{lin2022gradient}, GFM\textsuperscript{+} \citep{chen2023faster} and ZOO2N \citep{kornowski2024algorithm} on LIBSVM datasets ``a9a'',  ``covtype'', and ``w8a'' \citep{chang2007LIBSVM}.
We set the smoothing parameter to $\delta' = 0.001$ and tune the stepsize over the set $\{0.1, 0.03, \dots, 3\times 10^{-7}, 1\times 10^{-7}\}$ for all methods.
For GFM\textsuperscript{+}, all other hyperparameters are set to their default values.
For both ZOO2N and ZOCOON, we select the parameter $D$ from $\{1\times 10^{-2}, 3\times 10^{-3}, \dots, 1\times 10^{-5}\}$. In addition, we set the gradient clipping parameter to $\tau = 1\times 10^{-2}$ for ZOCOON.
We run each algorithm with 10 different random seeds on each dataset and report the results in Figure~\ref{fig:reg_svm}. The vertical axis shows the mean loss averaged over the runs, while the horizontal axis denotes the zeroth-order function query complexity.
We observe that ZOCOON and ZOO2N converge significantly faster than the other baseline methods. Moreover, ZOCOON exhibits more stable performance than ZOO2N on most datasets, which is consistent with our theoretical results under the heavy-tailed noise setting.

\section{Conclusion}\label{sec:conclusion}
In this work, we propose ZOCOON, a stochastic zeroth-order method for solving stochastic nonconvex nonsmooth optimization problems under heavy-tailed noise. We establish both in-expectation and high-probability convergence guarantees for ZOCOON. We further show that the stochastic zeroth-order complexity of our method matches that of the best-known zeroth-order methods under the bounded-variance assumption. Moreover, by establishing a lower bound of oracle complexity of stochastic first-order methods for nonconvex nonsmooth optimization, the dependence of ZOCOON on the accuracy parameters $\delta$ and $\epsilon$ matches the lower bound.

For future work, it would be interesting to investigate whether the dependence of our algorithm on the dimension $d$ is optimal. Another promising direction is to extend our framework to constrained nonconvex nonsmooth optimization with heavy-tailed noise.

\bibliographystyle{plainnat}
\bibliography{reference}

\newpage
\appendix The appendix is organized as follows. 
Section \ref{appendix:establish} introduces several key lemmas that are essential for the convergence analysis of the ZOCOON method.
Section \ref{appendix:proof_zocoon} presents the proof of the in-expectation and high-probability convergence guarantee of the ZOCOON method for solving the nonconvex nonsmooth optimization problem.

\section{Supporting Lemmas}\label{appendix:establish}
In this section, we revisit some key lemmas that are essential for the convergence analysis of the proposed algorithms.

\begin{lemma}[\citet{kornilov2023accelerated}]
    Assumption \ref{asm:Lip} implies that $f(x)$ is $L$-Lipschitz on $\BR^d$.
\end{lemma}



In addition, we recall the following results for the clipping stochastic gradient.

\begin{lemma}[\citet{liu2024high}]
    Suppose $g$ is a heavy-tailed random vector, and it satisfies $\Norm{\BE[g]} \leq L$ and $\BE[\Norm{g - \BE[g]}^p] \leq \sigma^p$ for some $p \in (1, 2]$.
    Let $\hat{g}$ be a truncated gradient with a positive clipping parameter $\tau$:
    \begin{align*}
     \hat{g} =
    \begin{cases}
        \min\left\{1, {\tau}/{\Norm{g}}\right\}g, & g\neq0, \\
        0, & g=0,
    \end{cases}
\end{align*}
then we have
\begin{align*}
  &   \Norm{ \BE[\hat{g}] - \BE[g]} \leq \frac{2^{p-1} (\sigma^p + L^p)}{\tau^{p-1}}, \\
   & \BE [\Norm{\hat{g}}^2] \leq 2^{p-1} \tau^{2 - p} (\sigma^p + L^p). 
\end{align*}
\label{lemma:clip_grad}
\end{lemma}

\section{Convergence Analysis of ZOCOON method}\label{appendix:proof_zocoon}
In this section, we provide both in-expectation and high-probability convergence guarantees of the ZOCOON method for solving nonconvex nonsmooth optimization problems.

\subsection{In-Expectation Guarantee}
We provide the proof of Lemma~\ref{thm:o2n_expect} as follows.

\begin{proof}
    Let $\Delta_n = x_{n} - x_{n-1}$, then we have
    \begin{align*}
        & h(x_M) - h(x_0) \\ 
        =& \sum_{n=1}^M \langle \BE_{w_n, z_n}[g_n], \Delta_n\rangle \\
        = & \underbrace{\sum_{n=1}^M \langle \hat{g}_n, \Delta_n - u_n \rangle}_{A_1} + \underbrace{\sum_{n=1}^M \langle \BE_{z_n}[g_n], u_n \rangle}_{A_2} + \underbrace{\sum_{n=1}^M \langle \BE_{w_n, z_n}[g_n] - \BE_{z_n}[g_n], \Delta_n \rangle}_{A_3}\\
       & + \underbrace{\sum_{n=1}^M \langle \BE_{z_n}[g_n] - \BE_{z_n}[\hat{g}_n], \Delta_n - u_n \rangle}_{A_4} + \underbrace{\sum_{n=1}^M \langle \BE_{z_n}[\hat{g}_n] - \hat{g}_n, \Delta_n - u_n \rangle}_{A_5},
    \end{align*}
    where $\{u_n\}$ is an arbitrary sequence of points whose norms are bounded by $D$.
    Specifically, we take
    \begin{align*}
        u_{(k-1)T + 1} = \dots = u_{kT} = u^k = - D \frac{\sum_{t=1}^T \nabla h(w_t^k)}{\Norm{\sum_{t=1}^T \nabla h(w_t^k)}}
    \end{align*}
    for $\forall k \in [K]$ throughout the analysis of the theorem. 
    In the following, we analyze each term separately.

    For the first term $A_1$, our choice of $\{u_n\}$ implies that
    \begin{align*}
        \sum_{n=1}^M \langle \hat{g}_n, \Delta_n - u_n\rangle = \sum_{k=1}^K \sum_{t=1}^T \langle \hat{g}_t^k, \Delta_t^k - u^k \rangle,
    \end{align*}

where $\hat{g}_t^k = \hat{g}_{(k-1)T + t}$ and $\Delta_t^k = \Delta_{(k-1)T + t}$.
We can interpret this term as optimizing the online linear functions $\ell_n(x) = \langle \hat{g}_n, x \rangle$ with $K$ restarts.
The standard OGD-style update of $\Delta_n$ with the fixed step size $\eta$ has the regret of
\begin{align*}
    \sum_{t=1}^T \ell_t(w_t) - \ell_t(u) \leq \frac{1}{2\eta} \Norm{u - w_1}^2 + \frac{\eta}{2} \sum_{t=1}^T\Norm{\nabla \ell_t(w_t)}^2
\end{align*}
for an arbitrary sequence of convex functions $\ell_t$, domain bounded sequence $\{w_t\}$ and $u$, according to \citet[Theorem 2.13]{orabona2019modern}.
Accordingly, we can show that
\begin{align*}
    \sum_{t=1}^T \langle \hat{g}_t^k, \Delta_t^k - u^k \rangle \leq \frac{2}{\eta} D^2 + \frac{\eta}{2} \sum_{t=1}^T \Norm{\hat{g}_t^k}^2
\end{align*}
for $T$ rounds within the $k$-th restart.
Note that $\Norm{\hat{g}_t^k} \leq \tau$ and $\BE_{z_n}[\Norm{\hat{g}_t^k}^2] \leq 2^{p-1} \tau^{2-p}(\sigma^p + L^p)$ according to Lemma \ref{lemma:clip_grad}, which implies that
\begin{align*}
    \sum_{t=1}^T \BE[ \langle \hat{g}_t^k, \Delta_t^k - u^k \rangle] \leq \frac{2}{\eta} D^2 + {2^{p-2} \tau^2 \eta} \leq \frac{2}{\eta} D^2 + { \tau^2 \eta},
\end{align*}
where the first inequality is due to the choice of $\tau = T^{1/p}(\sigma^p + L^p)^{1/p}$. 
Let $\eta = D / \tau$, we have
\begin{align*}
    \BE[A_1] = \sum_{k=1}^K \sum_{t=1}^T \BE[\langle \hat{g}_t^k, \Delta_t^k - u^k \rangle] \leq 3 D K \tau.
\end{align*}

For the term $A_2$, our choice of $u_n$ implies that
\begin{align*}
    \BE[A_2]  = \sum_{n=1}^M \langle \BE_{z_n}[g_n], u_n \rangle = - \BE\left[ D T \sum_{k=1}^K \Norm{\frac{1}{T}\sum_{t=1}^T \nabla h(w_t^k)}\right].
\end{align*}

For the term $A_4$, we can show that
\begin{align*}
    \BE[A_4] = \sum_{n=1}^M \langle \BE_{z_n}[g_n] - \BE_{z_n}[\hat{g}_n], \Delta_n - u_n \rangle \leq &  \sum_{n=1}^M \Norm{\BE_{z_n}[g_n] - \BE_{z_n}[\hat{g}_n]}\Norm{\Delta_n - u_n} \\
    \leq & DKT \frac{2^p (\sigma^p + L^p)}{\tau^{p-1}}.
\end{align*}
where the first inequality is due to Cauchy--Schwarz inequality, and the second inequality follows from Lemma \ref{lemma:clip_grad}.

For the terms $A_3$ and $A_5$, it is straightforward to see that
\begin{align*}
    \BE[A_3] = \BE[A_5] = 0.
\end{align*}
Putting everything together, we can show that
\begin{align*}
    \frac{1}{K}\sum_{k=1}^K\BE \left[\Norm{\frac{1}{T}\sum_{t=1}^T \nabla h(w_t^k)} \right] \leq &\frac{\BE[h(x_0) - h(x_M)]}{D M} + \frac{3 \tau}{T} + \frac{2^p (\sigma^p + L^p)}{\tau^{p-1}} \\
    \leq& \frac{\BE[h(x_0) - h(x_M)]}{D M} + \frac{8 \tau}{T}\\
     = & \frac{\BE[h(x_0) - h(x_M) ] T}{\delta M} + \frac{8 (\sigma^p + L^p)^{\frac{1}{p}}}{T^{1 - \frac{1}{p}}}.
\end{align*}

By taking
\begin{align*}
T = \delta^{\frac{p}{2p-1}} M^{\frac{p}{2p-1}}(\sigma^p + L^p)^\frac{1}{2p-1}\BE[h(x_0) - h(x_M)]^{-\frac{p}{2p-1}},
\end{align*}
we conclude with
\begin{align*}
    \BE[\Norm{\nabla h(w_{\rm out})}_{\delta'}] \leq & \frac{1}{K}\sum_{k=1}^K \BE\left[\Norm{ \nabla h(\bar{w}^k)}\right] \\ 
    = &\frac{1}{K}\sum_{k=1}^K \BE\left[\Norm{\frac{1}{T}\sum_{t=1}^T \nabla h(w_t^k)}\right] \\ 
    \leq & \max\left(\frac{9\BE[h(x_0) - h(x_M) ]^{\frac{p-1}{2p-1}} (\sigma^p + L^p)^{\frac{1}{2p-1}}}{(\delta M)^{\frac{p-1}{2p-1}}}, \frac{16 (\sigma^p + L^p)^{\frac{1}{p}}}{M^\frac{p-1}{p}} \right) + \frac{h(x_0) - h(x_M)}{\delta M}.
\end{align*}
where the first inequality is due to Jensen's inequality.
\end{proof}

\subsection{High-Probability Guarantee}\label{appendix:high_prob_res}
Before presenting the high-probability convergence guarantee of the ZOCOON algorithm, we first establish the complexity result for the clipped online-to-nonconvex method with a general gradient estimator (Algorithm~\ref{alg:o2n}), stated as follows.
\begin{lemma}[\citet{liu2024high}]
    Let $\delta', \epsilon \in (0, 1)$, and let $h: \BR^d \to \BR$ be an $L$-Lipschitz function and $h^* = \inf_{x \in \BR^d} h(x)$ such that $h(x_0) - h^* \leq \Delta_h$.
    Suppose that ${\rm GradientEstimator}(x, \delta', \xi)$ returns an unbiased gradient estimator of $h(x)$ which satisfies that
    \begin{align*}
       \BE[{\rm GradientEstimator}(x, \delta', \xi) ]=  \nabla h(x) \quad {\rm and} \quad \BE[\Norm{{\rm GradientEstimator}(x, \delta', \xi) - \nabla h(x)}^p] \leq \sigma^p.
    \end{align*}
    By setting parameters
    \begin{align*}
        &\tau = T^{\frac{1}{p}}(\sigma^p + L^p)^{\frac{1}{p}} \left(\log\left(\frac{2 K}{q}\right)\right)^{- \frac{1}{p}}, ~~~ T = \min\left(\ceil{\left(\frac{A M \delta'}{h(x_0) - \inf h(x)}\right)^{\frac{p}{2p-1}}}, \frac{M}{2}\right),  ~~~ K = \floor{\frac{M}{T}}, ~~~D = \frac{\delta'}{T},
    \end{align*}
    then with probability at least $1-q$, Algorithm \ref{alg:o2n} ensures
    \begin{align*}
        \frac{1}{K} \sum_{k=1}^K \Norm{\nabla h(\Bar{w}^k)}_{\delta'} \leq \frac{2 \Delta_h}{\delta' M} + \frac{2 B}{\sqrt{M}} + \max\left(\frac{4A^{\frac{p}{2p-1}} \Delta_h^{\frac{p-1}{2 p - 1}}} {(\delta' M)^{\frac{p-1}{2p-1}}}, \frac{4A}{M^{\frac{p-1}{p}}}\right),
    \end{align*}
    where 
    \begin{align*}
        A = (\sigma^p + L^p)^{\frac{1}{p}} \log\left(\frac{2K}{q}\right)^{- \frac{1}{p}} \left( 8 + \frac{29}{2} \log \left(\frac{2K}{q} \right)\right), ~~~ B = \sqrt{8 \log\left(\frac{6}{q}\right)}.
    \end{align*}
    \label{thm:o2n_high_prob}
\end{lemma}

We next establish the high-probability convergence guarantee of the ZOCOON algorithm through the proof of Theorem~\ref{thm:high_prob_res}.

\begin{proof}
According to Lemma~\ref{lemma:zero_order_prop}, Algorithm \ref{alg:grad} with $\xi\sim\fP$ will output the gradient estimate
\begin{align*}
g = \frac{d}{2 \delta'}(F(x + \delta' w; \xi) - F(x - \delta' w; \xi))w    
\end{align*}
that satisfies $\BE[g] = \nabla f_{\delta'}(x)$ and $\BE[\Norm{g-\nabla f_{\delta'}(x)}^p]= \fO(d^{\frac{p}{2}}L^p)$.

Then, running Algorithm~\ref{alg:o2n} by taking  Algorithm~\ref{alg:grad} as the subroutine  
${\rm GradientEstimator}(x, \delta', \xi)$, 
which satisfies the conditions
$\BE[{\rm GradientEstimator}(x, \delta', \xi) ] = \nabla h(x)$ and 
\begin{align*}
     \BE[\Norm{{\rm GradientEstimator}(x, \delta', \xi) - \nabla h(x)}^p] \leq \sigma^p,
\end{align*}
where $h = f_{\delta'}$ and $\sigma = \fO(\sqrt{d} L)$.
In the view of Lemma~\ref{thm:o2n_expect} by taking $\delta'=\delta/2$, then with probability at least $1-q$, it holds that
\begin{align}\label{high_prob_converg_main}
\begin{split}
       \Norm{\nabla f_{\delta / 2} (w_{\rm out})}_{\delta / 2} \leq \frac{1}{K} \sum_{k=1}^K \Norm{\nabla f_{\delta/2}(\Bar{w}^k)}_{\delta/2} \leq \frac{4 \Delta_{f_{\delta/2}}}{\delta M} + \frac{2 B}{\sqrt{M}} + \max\left(\frac{4 A^{\frac{p}{2p-1}} (\Delta_{f_{\delta/2}})^{\frac{p-1}{2 p - 1}}} {(\delta M / 2)^{\frac{p-1}{2p-1}}}, \frac{4A}{M^{\frac{p-1}{p}}}\right) \coloneqq \epsilon',
\end{split}
\end{align}
 where 
    \begin{align*}
        A = (d^{\frac{p}{2}} L^p + L^p)^{\frac{1}{p}} \log\left(\frac{2K}{q}\right)^{- \frac{1}{p}} \left( 8 + \frac{29}{2} \log \left(\frac{2K}{q} \right)\right), ~~~ B = \sqrt{8 \log\left(\frac{6}{q}\right)}, ~~~\Delta_{f_{\delta/2}}=f_{\delta/2}(x_0) - \inf_{x \in \BR^d} f_{\delta/ 2}(x).
    \end{align*}
Hence, the output $w_{\rm out}$ generated from Algorithm \ref{alg:o2n} is a $(\delta / 2, \epsilon')$ Goldstein stationary point of $f_{\delta/ 2}$.
Consequently, Lemma \ref{lemma:nonsmooth_smooth_conversion} implies the output $w_{\rm out}$ is a $(\delta, \epsilon')$-Goldstein stationary point of $f$ with probability at least $1-q$, 
\begin{align}\label{high_prob_converg_main2}
   & \Norm{\nabla f(w_{\rm out})}_{\delta}
   \leq  \epsilon'.
\end{align}
Comparing equations (\ref{high_prob_converg_main2}) with the desired result of Theorem~\ref{thm:o2n_expect}, we only need to show that $\Delta_{f_{\delta/2}}\leq\delta L + \Delta$.
This can be verified as follows
\begin{align*}
\Delta_{f_{\delta/2}} 
= &f_{\delta/2}(x_0) -  f_{\delta/2}^*\\
= & (f_{\delta/2}(x_0) - f(x_0)) + (f(x_0)- f^*) + (f^* -  f_{\delta/2}^*) \\
\leq &  \frac{\delta L}{2} + \Delta + \frac{\delta L}{2} = \delta L + \Delta,
\end{align*}
where the inequality is based on result $\Norm{f - f_{\delta/2}}_\infty \leq {\delta L}/{2}$ from Lemma~\ref{smooth_lemma} and the setting of $\Delta$.

\end{proof}

\end{document}